\documentclass[10pt,journal,compsoc]{IEEEtran}

\newcounter{docpart}

\newcounter{olddocpart}

\ifCLASSINFOpdf
 
\else
 
\fi
 
\usepackage{graphicx,subfigure}
\usepackage[colorlinks,linkcolor=black,anchorcolor=black,citecolor=black]{hyperref}

\usepackage{amsfonts,amsmath,amssymb,amsthm,bm}
\usepackage{algorithm,algorithmic}

\usepackage[numbers]{natbib}
\usepackage{multibib}
\newcites{app}{References}

\usepackage[english]{babel}
\usepackage{mdwlist}

\usepackage{comment}
\usepackage{cases}
\usepackage{multirow}

\usepackage{booktabs}

\usepackage{color}
\usepackage{tabularx}
 
\usepackage[table]{xcolor}

\DeclareMathOperator*{\argmax}{argmax}

\def\etc{\emph{etc}}
\def\etal{\emph{et al.}}
\def\eg{\emph{e.g.}}
\def\ie{\emph{i.e.}}
\def\trans{^{\rm T}}

\def\calL{\mathcal{L}}
\def\calS{\mathcal{S}}

\def\f{\mathbf{f}}

\def\x{\mathbf{x}}

\newtheorem{observation}{Observation}

\definecolor{Gray}{gray}{0.85}
\definecolor{LightCyan}{rgb}{0.88,1,1}

\newcommand{\tabincell}[2]{\begin{tabular}{@{}#1@{}}#2\end{tabular}}
\begin{document}
\title{Domain Agnostic Learning for Unbiased Authentication}

\author{
\IEEEauthorblockN{Jian Liang, Yuren Cao, Shuang Li, Bing Bai, Hao Li, Fei Wang, Kun Bai}

\IEEEcompsocitemizethanks{\IEEEcompsocthanksitem J. Liang is with AI for international Department, Alibaba Group,
Beijing, 100102, China.\protect\\
E-mail: liangjianzb12@gmail.com.
\IEEEcompsocthanksitem Y. Cao is with the Cloud and Smart Industries Group, Tencent, Beijing, 100089, China.\protect\\
E-mail: laurenyrcao@tencent.com.
\IEEEcompsocthanksitem S. Li is with the school of Computer Science and Technology, Beijing Institute of Technology, Beijing, 100081, China.\protect\\
E-mail: shuangli@bit.edu.cn.
\IEEEcompsocthanksitem B. Bai is with the Cloud and Smart Industries Group, Tencent, Beijing, 100089, China.\protect\\
E-mail: baibing12321@163.com.
\IEEEcompsocthanksitem H. Li is with the Cloud and Smart Industries Group, Tencent, Beijing, 100089, China.\protect\\
E-mail: leehaoli@tencent.com.
\IEEEcompsocthanksitem F. Wang is with the Department of Population Health Sciences, Weill Cornell Medical College, New York, NY, 10065, USA.\protect\\
E-mail: few2001@med.cornell.edu.
\IEEEcompsocthanksitem K. Bai is with the Cloud and Smart Industries Group, Tencent, Guangzhou, 510630, China.\protect\\
E-mail: kunbai@tencent.com.
}%
\thanks{Manuscript received April 19, 2005; revised August 26, 2015.}
}

\markboth{IEEE Transactions on Pattern Analysis and Machine Intelligence}%
{Shell \MakeLowercase{\textit{et al.}}: Bare Demo of IEEEtran.cls for Computer Society Journals}

\IEEEtitleabstractindextext{%
\begin{abstract}
Authentication is the task of confirming the matching relationship between a data instance and a given identity. Typical examples of authentication problems include face recognition and person re-identification. Data-driven authentication could be affected by undesired biases, i.e., the models are often trained in one domain (e.g., for people wearing spring outfits) while applied in other domains (e.g., they change the clothes to summer outfits). Previous works have made efforts to eliminate domain-difference. They typically assume domain annotations are provided, and all the domains share classes. However, for authentication, there could be a large number of domains shared by different identities/classes, and it is impossible to annotate these domains exhaustively. It could make domain-difference challenging to model and eliminate. In this paper, we propose a domain-agnostic method that eliminates domain-difference without domain labels. We alternately perform latent domain discovery and domain-difference elimination until our model no longer detects domain-difference. In our approach, the latent domains are discovered by learning the heterogeneous predictive relationships between inputs and outputs. Then domain-difference is eliminated in both class-dependent and class-independent spaces to improve robustness of elimination. We further extend our method to a meta-learning framework to pursue more thorough domain-difference elimination. Comprehensive empirical evaluation results are provided to demonstrate the effectiveness and superiority of our proposed method. 
\end{abstract}

}

\maketitle

\IEEEdisplaynontitleabstractindextext
\IEEEpeerreviewmaketitle

\section{Introduction}\label{sec:introduction}

\IEEEPARstart{A}{uthentication} is the problem of confirming whether the data instances match personal identities. There is a variety of authentication applications including face recognition~\cite{zhao2003face}, fingerprint verification~\cite{yager2004fingerprint} and person re-identification~\cite{bedagkar2014survey,zheng2016person}. However, the data-driven authentication process often suffers from undesired biases. In particular, the verification model is usually trained in one domain and tested and verified in other domains, which {could cause inconsistent prediction results due to domain difference/shift.} For example, for person re-identification~\cite{bedagkar2014survey}, the prediction could be compromised due to the seasonal outfits changing or the angle variation between a camera and a pedestrian.\footnote{The seasonal outfits include four domains: \texttt{spring}, \texttt{summer}, \texttt{autumn}, and \texttt{winter}. The outfit and the shooting angle can be regarded as two \emph{types} of domain-difference.} {Domain difference/shift can take many forms, including covariate shift (distribution difference in $p(\x)$, where $\x$ denotes the feature)~\cite{shimodaira2000improving}, target/prior probability shift (difference in $p(y)$, where $y$ denotes the output target)~\cite{zhang2013domain,storkey2009training}, conditional shift (difference in $p(y\mid \x)$)~\cite{zhang2013domain}, 
and joint shift (difference in $p(y,\x)$)~\cite{moreno2012unifying}.
Domain transfer methods can be categorized into two types~\cite{day2017survey,friedjungova2017asymmetric}: 1) symmetric methods that unify multiple domains into one common space; 2) asymmetric methods that map data from one domain to another. }
To understand how we can alleviate the aforementioned problem, we study the learning task for unbiased authentication. Specifically, we treat authentication as a recognition problem so that each identity corresponds to a class. {For model efficiency, this paper focuses on the symmetric methods, eliminating domain-difference to unify domains.}



%

\begin{table}\small
  \caption{An example of the assumptions of our {proposed GCLDR problem. Each class set includes its unique classes. ``Train''/``Test'' denotes the data for training/testing. ``Latent'' suggests domain labels are absent.}}
  \label{tab:proposed_problem}
  \centering
  \begin{tabular}{lccc}
    \toprule
    &  \tabincell{c}{Class Set 1} &\tabincell{c}{Class Set 2} & \tabincell{c}{Class Set 3}\\
    \midrule
    Latent Domain 1 & Train   & Test\cellcolor{Gray} & Test\cellcolor{Gray}\\
Latent Domain 2 & Test \cellcolor{Gray}& Train& Test\cellcolor{Gray}  \\
Latent Domain 3 & Test \cellcolor{Gray}& Test\cellcolor{Gray}& Train  \\
    \bottomrule
  \end{tabular}
\end{table}

The existing research on domain generalization~\cite{xu2014exploiting,khosla2012undoing,li2017deeper,ghifary2015domain,muandet2013domain,li2018learning,li2018deep,li2018domain} or multi-domain adaptation~\cite{sun2011two,duan2009domain,mansour2009domain,duan2009domain,duan2012domain} typically aims at learning 
{domain-transfer} from multiple training domains. {One limitation of} these approaches is that they assume domain labels are available. However, in real applications, it is labor-intensive and time-consuming to provide annotations of all domains, especially when the number of domains is massive. 
Therefore, researchers recently propose to detect the {\em latent} domains whose labels are absent~\cite{hoffman2012discovering,li2017domain,gong2013reshaping,mancini2018boosting,xiong2014latent,niu2015multi,mancini2019discovering}. 
These methods are typically purely based on features. However, the original reason why domain generalization/adaptation is essential is that the learned predictive relationship between features and targets, modeled by $p(y\mid \x)$, on training data might change on testing data.
Therefore, the key to understand domain difference is why the predictive relationship $p(y\mid \x)$ is different in different domains. 
Consequently, to {precisely} learn and unify latent domains, we should exploit the heterogeneity in $p(y\mid \x)$, in addition to the heterogeneity in features, as in most of the existing research.




\begin{figure}[t]
\centering
\subfigure[Training]{\includegraphics[width=0.8\linewidth]{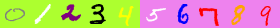}}
\subfigure[Testing]{\includegraphics[width=0.8\linewidth]{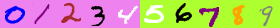}}
\caption{An experiment setting of C-MNIST with the background color as the domain-difference. Best viewed in color.}
\label{fig:cmnist_dataset_bg}
\end{figure}

Another limitation of existing approaches is that they require the classes for recognition to be shared across all the domains. 
However, since the classes are the identities, it is impossible to collect all domain data for every individual in the authentication task. The real scenario is that only the data from one domain are collected for each subset of individuals, and this phenomenon is characterized in the generalized cross-domain recognition (GCDR) problem~\cite{liang2019additive}, although they assume that domain labels are provided. In this paper, we propose to study the problem where domain labels are absent, which is referred to as a generalized cross-latent-domain recognition (GCLDR) problem. {We present a toy example with only one type of latent domain difference in Table~\ref{tab:proposed_problem}. Unlike the setups for standard domain generalization or adaptation, for training data in our example, different latent domains do not share any classes. For a colored-digit-recognition example shown in Fig.~\ref{fig:cmnist_dataset_bg}, in the training data set, one can only observe images of digits $0\sim 4$ with a green background, and digits $5\sim9$ with a pink background, while the images of $5\sim9$ with the green background are not observed. The background colors are the latent domains and do not share any digit in the training set. Consequently, for any testing data sample, 
a trained model could generate wrong prediction due to the misleading of its domain information. 
Thus, we should transfer knowledge across domains, then the recognition model trained on other latent domains can be used.

}

To address the above issues, we propose a novel domain-agnostic method to learn and unify latent domains to tackle the GCLDR problem. Specifically, we propose a latent-domain discovery (LDD) module
to 
capture the heterogeneous predictive relationships between features and targets from different latent domains, where each predictive relationship comes from one latent domain. The LDD model also includes a domain-discrimination component, which discriminates latent domains based on features. The posterior distribution of latent domains given features and targets naturally integrates the recognition and domain-discrimination components via the Bayes rule such that{ $p(z\mid y,\x)\propto p(y\mid \x)p(z\mid \x)$, where $z$ denotes the latent domain. Thus, leveraging the predictive relationships $p(y\mid \x)$s, we discover latent domains by the joint distribution $p(y,\x)$}. By forcing the posterior probability of every latent domain to be equal, we can eliminate domain difference by the joint distribution { $p(y,\x)$}.
We alternately perform the latent-domain discovery and unification processes. Therefore, every possible type of domain difference (\eg, season, and shooting angle) can be learned and eliminated successively.
As a consequence, the number of latent domains, which is a hyper-parameter, can be robustly fixed to be two: as long as separated latent domains exist, they can be organized into two different groups. On the other hand, inspired by Liang~\etal~\cite{liang2019additive}, we propose to eliminate latent domain-difference in both class-dependent and class-independent spaces in two branches of our network, respectively. This architecture can be more robust for domain-difference elimination. 
Finally, inspired by Li~\etal~\cite{li2018learning}, we provide a meta-learning extension of our method to encourage commonality among latent domains by splitting latent domains into meta-train domains and meta-test domains and optimizing via a learn-to-learn procedure: learning how to minimize the losses on the meta-train domains to minimize the losses on the meta-test domains~\cite{li2018learning}.
The experimental results on benchmark and real-world data sets demonstrate the effectiveness and superiority of our method. We also conduct ablation experiments to show the contribution of each component of our proposed framework.
 
Our contributions are highlighted as follows.

\begin{itemize}
    \item We propose to address a generalized cross-latent-domain recognition (GCLDR) problem, where domain labels are absent, and domains do not share classes. This problem is challenging and common in real applications of authentication, \eg, face recognition authentication.
    \item We provide valuable insights that 1) alternately performing latent-domain discovery and domain-difference elimination can boost generalization to unseen $\langle$class, domain$\rangle$ combinations; 
    2) both procedures should be based on predictive relationships between features and targets via posterior probabilities, which is novel compared with existing domain adaptation/generalization methods; and 3) eliminating domain-difference in both class-dependent and class-independent spaces can boost robustness of domain-difference elimination.
    \item Our method achieves significant improvements compared with baselines on one benchmark and two authentication datasets, and even rivals the state-of-the-art methods using extra domain labels on both authentication datasets.
\end{itemize}
  
The remainder of this paper is organized as follows. Section~\ref{sec:related_work} discusses related works. The proposed methods are presented in Section~\ref{sec:method}. Section~\ref{sec:exp} presents an empirical evaluation of the proposed approaches. Section~\ref{sec:discussion} discusses why related works failed in our setting, and Section \ref{sec:conclusion} provides conclusions.

\section{Related Works}\label{sec:related_work}

\noindent\textbf{Domain Generalization/Adaption \;} Domain generalization approaches~\cite{xu2014exploiting,khosla2012undoing,li2017deeper,ghifary2015domain,muandet2013domain,li2018learning,li2018deep,li2018domain,shankar2018generalizing} typically train models on single/multi-domain data with shared classes for recognition on an unseen domain, while domain adaptation approaches~\cite{patel2015visual,ganin2015unsupervised,bousmalis2016domain,long2017deep,csurka2017domain,sun2015survey,wang2018deep,long2018conditional,cicek2019unsupervised,wen2019bayesian,dai2019contrastively,chen2019joint,Lee_2019_ICCV} typically train models on source domains and recognize on target domains which share classes with source domains without class labels.  Domain generalization and adaptation both typically assume that on training data, classes are shared across domains and domain labels are provided, which do not hold in our GCLDR problem.  

\noindent\textbf{Domain Agnostic Learning \;} Recently, several domain-agnostic learning approaches~\cite{peng2019domain,chen2019blending,liu2019compound,romijnders2019domain} emerge to typically handle the domain-adaptation problem where the target domain may contain several sub-domains without domain labels~\cite{peng2019domain}. DADA~\cite{peng2019domain} and  OCDA~\cite{liu2019compound} propose novel mechanisms and achieve effective domain-adaptation performances, but do not discover latent domains in the target domain and exploit the information. By contrast, BTDA~\cite{chen2019blending} clusters raw and deep features to discover latent domains. However, its latent-domain discovery is only feature-based and does not exploit the heterogeneous predictive relationships of $p(y\mid \x)$. DANL~\cite{romijnders2019domain} learns a normalization layer, but may be limited for more sophisticated domain-difference~\cite{peng2019domain}. Except for BTDA, these methods work on the extra knowledge of domain labels (source/target).

\noindent\textbf{Latent Domain Discovery \;} Existing \emph{explicit} latent-domain discovery approaches~\cite{hoffman2012discovering,li2017domain,gong2013reshaping,mancini2018boosting,xiong2014latent,niu2015multi,mancini2019discovering} typically build special models to learn latent domains {explicitly} based on features only. As an exception, Xiong~\etal~\cite{xiong2014latent} propose to learn latent domains via a conditional distribution $p(z\mid  y,\x)$, which is not based on the predictive relationship $p(y\mid \x)$, but based on linear addition of deep features of $y$ and $\x$. When latent domains are discovered, a main-stream of these approaches does not perform domain transfer. However, as explained in introduction, it is not appropriate for our GCLDR problem. 
In contrast, mDA~\cite{mancini2018boosting} (an improved version of DANL~\cite{romijnders2019domain}) and its improved version CmDA~\cite{mancini2019discovering} unify domains by normalizing the hidden space for each domain to have zero mean and unit standard deviation, which may be ineffective for more sophisticated domain-difference. In addition to the above explicit discovery methods, there are also several methods that learn latent domains \emph{implicitly}, including ML-VAE~\cite{bouchacourt2017multi}, MCD~\cite{saito2018maximum} and MCD-SWD~\cite{lee2019sliced}, which may encounter sub-optimal solutions due to not explicitly modeling multiple latent domains thus may ignore fine-grained information.
 
\noindent\textbf{Meta Learning \;}
Meta-learning's idea is learning to learn~\cite{thrun2012learning,schmidhuber1997shifting}, which recently gains great popularity~\cite{finn2017model,sariyildiz2019gradient,vinyals2016matching,li2017deeper,ravi2016optimization,andrychowicz2016learning}. Recently, a few meta-learning domain generalization approaches~\cite{li2018learning,balaji2018metareg,li2019feature} have been proposed, which learn a set of domains to learn another set of domains. However, these methods require domain labels for training. Nonetheless, inspired by these methods, we propose a meta-learning framework when latent domains are discovered.

\noindent\textbf{Self-Supervised Learning \;} Self-supervised learning typically formulates a auxiliary learning task~\cite{kolesnikov2019revisiting,gidaris2018unsupervised,oord2018representation,caron2018deep,lee2018making,gidaris2018unsupervised} to improve supervised-learning without class labels, and recently is found to be effective for generalization~\cite{liu2019self} or domain generalization/adaptation~\cite{carlucci2019domain,xu2019self}. However, it may suffer sub-optimal solutions in our GCLDR problem due to not explicitly modeling or unifying latent domains.

  \section{Methodology}\label{sec:method}

This section lays out the details of our proposed network by first defining notations and problem settings. Consider a data set $\mathcal{D} = \{(\x^i,y^i)\}_{i=1}^n$ consisting of $n$ independent samples. For the $i$th sample, $\x^i\in\mathbb{R}^d$ is a feature vector with $d$ dimensions, and $y^i\in\mathbb{Z}_+$ is a categorical class label of the recognition task. The data set contains no domain labels. In other words, the setting of our proposed GCLDR problem extends the GCDR problem \cite{liang2019additive} such that no domain labels are given.
Throughout the paper, we denote $[k]$ as the index set $\{1,2,\ldots,k\}$.

  



\subsection{Heterogeneous Predictive Relationships Discovery and Unification}\label{subsec:hcdd}

 
Our LDD module 
discovers multiple predictive relationships for $p(y\mid \x)$. The discovery process utilizes the hidden features of a deep neural network. Here, we assume that we determine to discover $k\in\mathbb{Z}_+$ latent domains. Note that $k=2$ is adequate, since we will discover and unify the rest of the latent domains successively. Given a hidden feature vector $\f^i$ for the $i$th data sample $\x^i \ (i\in[n])$, we build $k$ local recognition networks $R_{l}^{1},\ldots,R_{l}^{k}$ to learn $k$ conditional distributions of $p(y^i\mid \f^i)$. Each conditional distribution corresponds to a subset of samples, and is denoted by $p(y^i\mid \f,R_{l}^{r})$ which follows a categorical distribution: 
\begin{equation}\label{eq:post}\small
p(y^i \mid \f^i,R_{l}^{r}) = \prod_{j=1}^cp(y^i=j \mid \f^i,R_{l}^{r})^{I(y^i=j)},r\in[k],
\end{equation}
where $c$ denotes the number of classes.
We further build a domain-discrimination network $D$ to discriminate which domain does $\f^i$ belong to. Then $D$ aims to learn $p(z^i=r\mid \f,D)$ for all $r\in[k]$, where $z^i$ denotes the latent domain for current $(y^i,\f^i)$. Then via the Bayes rule, the posterior probability of $(y^i,\f^i)$ belongs to the $r$th domain is
\begin{equation}\label{eq:post}\small
\begin{split}
 \rho^{i,r}&=p(   z^i=r \mid y^i,\f^i,\{R_l^{r}\}_{r=1}^k,D)\\
 &=  
   \frac{p( z^i=r \mid \f^i,D)\prod_{j=1}^cp(y^i=j \mid \f^i,R_l^{r})^{I(y^i=j)}}{\sum_{r'=1}^kp(z^i=r' \mid \f^i,D)\prod_{j=1}^cp(y^i=j \mid \f^i,R_l^{r'})^{I(y^i=j)}}.
\end{split}
\end{equation}
The detailed derivation is in Appendix~\ref{sec:post_and_em}. {We can observe in Eq.~\eqref{eq:post} that, since $\f$ is based on $\x$, the posterior probability models $p(z\mid y,\x)$ which discover latent domains based on the joint distribution $p(y,\x)$. Therefore, given class-label information, the posterior probability can provide more accurate domain-discrimination than the feature-based discriminative probability $p( z^i=r \mid \f^i,D)$ which models $p(z\mid \x)$ using the information of $p(\x)$ only.}

We aim to provide an end-to-end optimization scheme so that we discover latent domains on each mini-batch of data samples in a common mini-batch based optimization procedure. Given the posterior probabilities $\{\rho^{i,r}\}$ --- soft selection of domains, we optimize:
\begin{equation}\label{eq:loss_ce}\small
\begin{split}
\ell_{d} = &-\frac{1}{b}\sum_{i=1}^b\sum_{r=1}^k\rho^{i,r}\sum_{j=1}^cI(y^i=j)\log p(y^i=j \mid \f^i,R_l^{r}) \\
&-\frac{1}{b}\sum_{i=1}^b\sum_{r=1}^k\rho^{i,r}\log p( z^i=r \mid \f^i,D),
\end{split}
\end{equation}
where $b$ denotes the batch size, which is resulted from the Expectation-Maximization derivation (see Appendix~\ref{sec:post_and_em} for details).


To unify latent domains, we propose to force the posterior probabilities to be equal across domains to {eliminate domain-difference by the joint distribution $p(y,\x)$}:
\begin{equation}\label{eq:loss_adv_all}\small
\begin{split}
\ell_{e} = & \frac{1}{b}\sum_{i=1}^b\sum_{r=1}^k( p(   z^i=r \mid y^i,\f^i,\{R_{l}^{r}\}_{r=1}^k,D) - 1/k)^2.
\end{split}
\end{equation}

Alternately computing the posteriors by Eq.~\eqref{eq:post} and minimizing losses in Eq.~\eqref{eq:loss_ce} and \eqref{eq:loss_adv_all}, we can discover and then unify all the latent domains successively, until our model no longer detects domain-difference.

 
\subsection{Double-Space Domain-Difference Elimination}\label{subsec:dsdde}

Based on the latent-domain discovery and unification module introduced in Section \ref{subsec:hcdd}, we propose to eliminate domain-difference both in a class-dependent space (where classes can be recognized) and a class-independent space (where classes cannot be recognized).




\begin{figure}[t]
\centering
\includegraphics[width=1\linewidth]{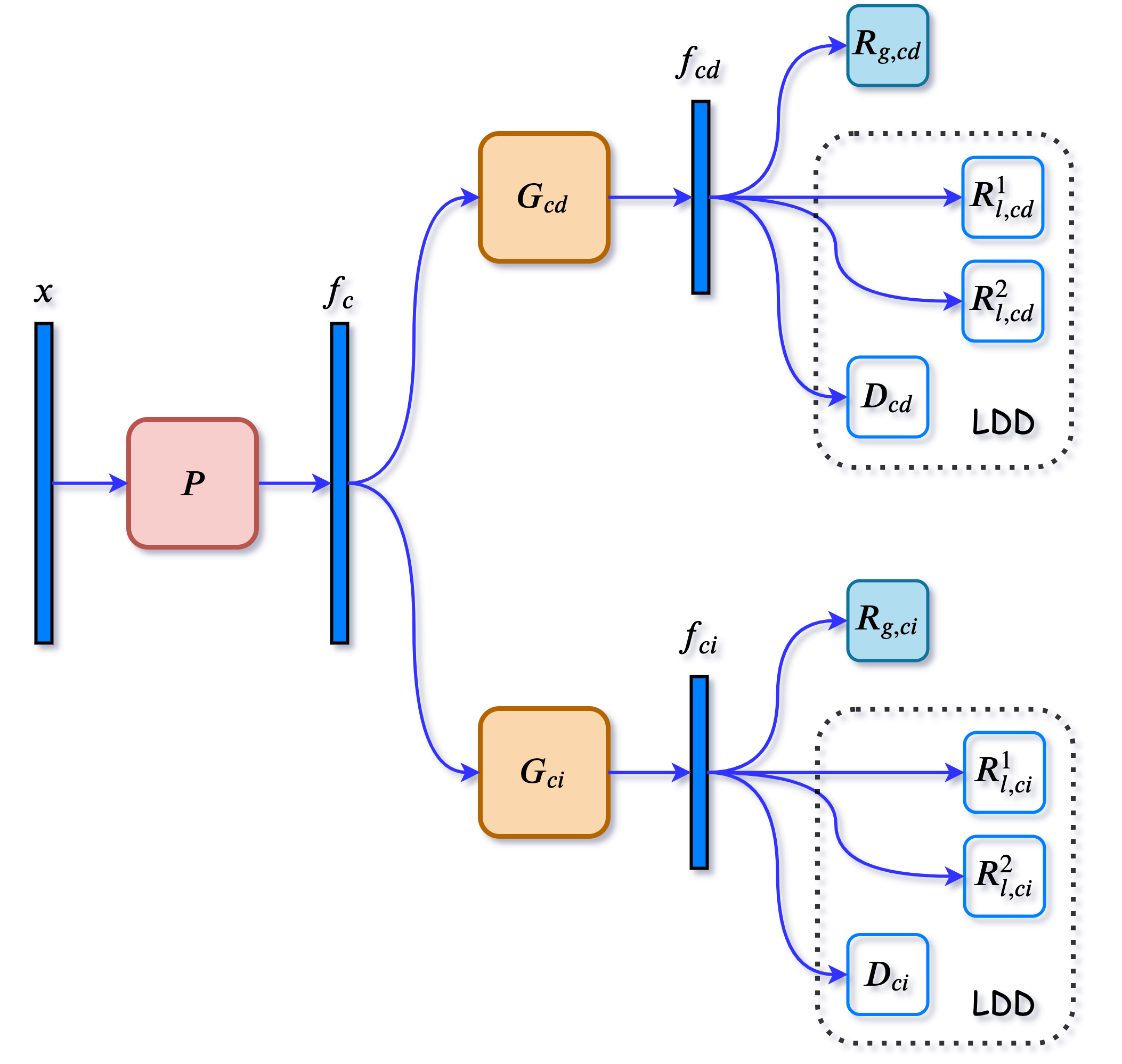}
\caption{The architecture of our framework. $P,G_{cd},G_{ci}$ are feature extractors. $R_{g,cd}$ learns class-dependent features $f_{cd}$s, while $R_{g,ci}$ learns class-independent features $f_{ci}$s. The LDD modules independently learn latent domains in each space. The posteriors of LDD are forced to be equal across domains to eliminate domain-difference.}
\label{fig:model_architecture}
\end{figure}

\begin{figure}[t]
\centering
\subfigure[Raw input images as $x$s]{\includegraphics[width=1\linewidth]{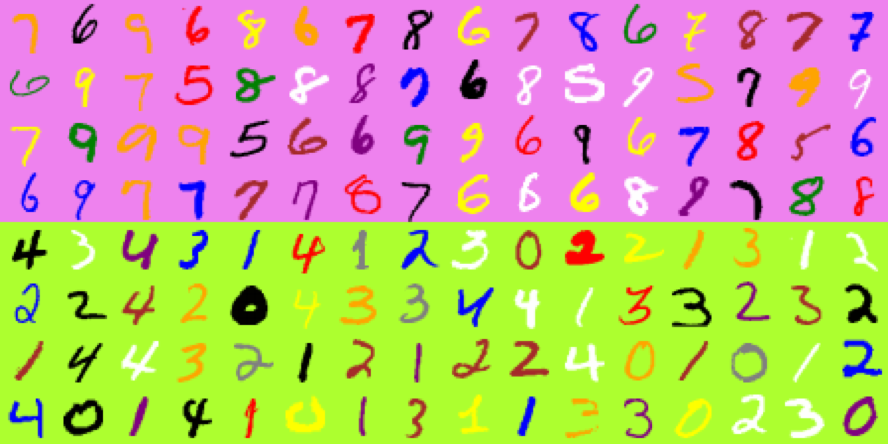}}
\subfigure[Generated hidden feature maps as $f_{cd}$s]{\includegraphics[width=1\linewidth]{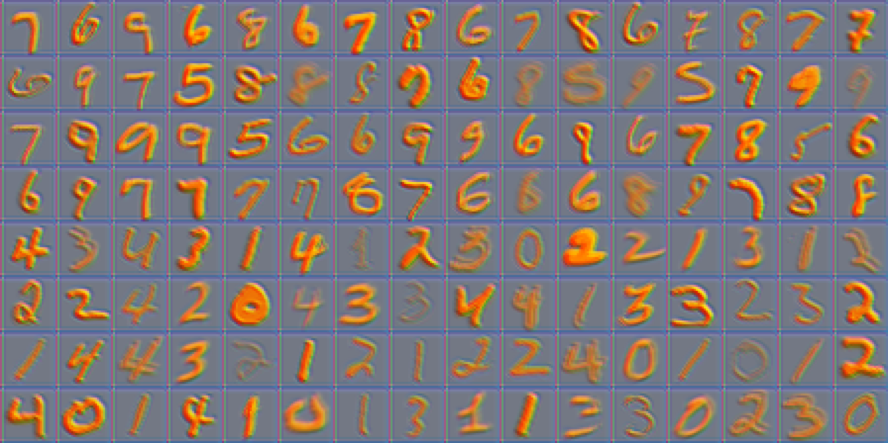}}
\caption{Examples of $f_{cd}$s on C-MNIST resulted from our method. Best viewed in color.}
\label{fig:cmnist_exp}
\end{figure}

We first introduce our model structure to learn hidden features in the above two spaces. As shown in Fig.~\ref{fig:model_architecture}, an input sample $\x^i$ is transformed by a mapping network $P$ into a hidden feature vector $\f_c$, which is further transformed by two feature-extraction networks $G_{cd},G_{ci}$ to obtain a class-dependent feature vector $\f_{cd}^i$ and a class-independent feature vector $\f_{ci}^i$, respectively. We let $\f_{cd}^i$ be class-dependent by using a global recognition network $R_{g,cd}$ to recognize the class from $\f_{cd}^i$ by minimizing:
\begin{equation}\label{eq:loss_ce_gobal_cd}\small
\begin{split}
\calL_{cd} &= \ell_c(\{\f_{cd}^i\}_{i=1}^b,R_{g,cd}),\\
\ell_c
&= -\frac{1}{b}\sum_{i=1}^b\sum_{j=1}^cI(y^i=j)\log p(y^i=j \mid \f_{\cdot}^i,R_{g,\cdot}).
\end{split}
\end{equation}
For $\f_{ci}^i$ to be class-independent, we learn by an adversarial learning process: we first learn a global recognition network $R_{g,ci}$ to recognize the class from $\f_{ci}^i$ by minimizing:
\begin{equation}\label{eq:loss_ce_gobal_ci}\small
\begin{split}
\calL_{ci} =\ell_c(\{\f_{ci}^i\}_{i=1}^b,R_{g,ci}),
\end{split}
\end{equation}
then learn the feature-extraction network $G_{ci}$ to eliminate class-dependence features by minimizing:
\begin{equation}\label{eq:loss_ce_gobal_ci_adv}\small
\begin{split}
&\calL_{ac} = \frac{1}{bc}\sum_{i=1}^b\sum_{j=1}^c( p(y^i=j \mid \f_{ci}^i,R_{g,ci}^{r})-p(y=j))^2,
\end{split}
\end{equation}
where $p(y=j)$ denotes the frequency of class $j$ on the training data.

The domain-difference elimination is also realized via adversarial learning. According to Section~\ref{subsec:hcdd}, we discover latent domains in both class-specific and class-independence spaces by minimizing
\begin{equation}\label{eq:loss_discover_cd}\small
\calL_d = \ell_{d}(\{\f_{cd}^i\}_{i=1}^b,\{R_{l,cd}^{r}\}_{r=1}^k,D_{cd})+ \ell_{d}(\{\f_{ci}^i\}_{i=1}^b,\{R_{l,ci}^{r}\}_{r=1}^k,D_{ci}),
\end{equation}
where $\ell_{d}$ is defined in Eq.~\eqref{eq:loss_ce}, $D_{\cdot}s$ are domain-discrimination networks, and $\{R_{l,\cdot}^{r}\}_{r=1}^k$s are groups of local recognition networks.
Then we eliminate domain-difference in both class-dependence and class-independence spaces. We fix the LDD modules and learn $P,G_{cd},G_{ci}$ by minimizing:
\begin{equation}\label{eq:loss_unify}\small
\begin{split}
\calL_{u} =  \ell_{e}(\{\f_{cd}^i\}_{i=1}^b,\{R_{l,cd}^{r}\}_{r=1}^k,D_{cd})+ \ell_{e}(\{\f_{ci}^i\}_{i=1}^b,\{R_{l,ci}^{r}\}_{r=1}^k,D_{ci}),
\end{split}
\end{equation}
where $\calL_{e}$ is defined in Eq.~\eqref{eq:loss_adv_all}.

We summarize our latent-domain discovery and unification in double spaces in Algorithm~\ref{alg:DA-UA}. For inference, we stack $P,G_{cd}$ and $R_{g,cd}$ to predict the class label $y^i$ for each sample $\x^i$.

\begin{algorithm} [t]
\caption{Learning Algorithm for GCLDR}
\label{alg:DA-UA}{\small
\begin{algorithmic}[1]
\REQUIRE{Data set $\mathcal{D} = \{(\x^i,y^i)\}_{i=1}^n$, where $\forall i\in[n], \ \x^i\in \mathbb{R}^{d}$ and $y^i\in[c]$. Number of latent domains $k\in \mathbb{Z}_+$. Batch size $b\in\mathbb{Z}_+$.}
\ENSURE{Recognition model: $R_{g,cd}(G_{cd}(P(\cdot)))$.}
\WHILE{not converge}
\STATE Sample a mini-batch $\{\x^i,y^i\}_{i=1}^b$.
\STATE Forward the min-batch to obtain $\f_{cd}^i=G_{cd}(P(\x^i))$ and $\f_{ci}^i=G_{ci}(P(\x^i))$ for all $i\in[b]$.
\STATE Compute posteriors based on $\{\f_{cd}^i\}_{i=1}^b$ and $\{\f_{ci}^i\}_{i=1}^b$ by Eq.~\eqref{eq:post}, respectively.
\STATE Optimize recognition and discrimination networks, \ie, $R_{g,cd},R_{g,ci},\{R_{l,cd}^{r}\}_{r=1}^k,\{R_{l,ci}^{r}\}_{r=1}^k,D_{cd},D_{ci}$:
\begin{equation}\label{eq:opt_pos}\small
    \min \ \calL_{cd} + \calL_{ci} + \calL_{d},
\end{equation}
where $\calL_{cd},\calL_{ci},\calL_{d,ci}$ are defined in Eq.~\eqref{eq:loss_ce_gobal_cd},~\eqref{eq:loss_ce_gobal_ci}, and~\eqref{eq:loss_discover_cd}, respectively.
\STATE Optimize mapping and feature-extraction networks, \ie, $P,G_{cd},G_{ci}$. \label{eq:post_merg}
\begin{equation}\label{eq:opt_neg}\small
    \min \ \calL_{cd} + \calL_{ac} +\calL_{u},
\end{equation}
where $\calL_{ac},\calL_{u}$ are defined in Eq.~\eqref{eq:loss_ce_gobal_ci_adv} and~\eqref{eq:loss_unify}, respectively.
\ENDWHILE
\end{algorithmic}}
\end{algorithm}

\subsection{Meta-Learning Paradigm Extension}\label{subsec:mlpe}
 
Recall that in Eq.~\eqref{eq:loss_adv_all},  when latent domains are discovered, class predictions between different domains are equalized to eliminate domain-differences on $p(y\mid \x)$. Inspired by Li~\etal~\cite{li2018learning}, we propose an alternative approach to encourage consistency of $p(y\mid \x)$ across latent domains by a meta-learning mechanism: we split latent domains into meta-train domains and meta-test domains, and learn how to minimize the losses on the meta-train domains in order to minimize the losses on the meta-test domains~\cite{li2018learning}. 

Specifically, for the loss of each domain, we use the local recognition loss soft-selected by the posteriors:
\begin{equation}\label{eq:loss_local_each}\small
\begin{split}
\ell_{s}(r) = -\frac{1}{b}\sum_{i=1}^b\rho^{i,r}\sum_{j=1}^cI(y^i=j)\log p(y^i=j \mid \f^i,R_l^{r}),
\end{split}
\end{equation}
where $\rho^{i,r}$ is the posterior defined in Eq.~\eqref{eq:post}. The loss in Eq.~\eqref{eq:loss_local_each} can be based on either LDD modules in class-dependent or class-independent spaces. So we merge the losses for both spaces:
\begin{equation}\label{eq:loss_local_each_both}\small
\begin{split}
\calL_s(r) =& \ell_{s}(r;\{\f_{cd}^i\}_{i=1}^b,\{R_{l,cd}^{r}\}_{r=1}^k,D_{cd})\\
&+ \ell_{s}(r;\{\f_{ci}^i\}_{i=1}^b,\{R_{l,ci}^{r}\}_{r=1}^k,D_{ci}).
\end{split}
\end{equation}
Then Step~\ref{eq:post_merg} in Algorithm~\ref{alg:DA-UA} can be replaced by Algorithm~\ref{alg:MLPU}.
\begin{algorithm} [t]
\caption{Meta-Learning for Conditional Unification}
\label{alg:MLPU}{\small
\begin{algorithmic}[1]
\REQUIRE{$\theta$ that collects the model-parameters for $P,G_{cd},G_{ci}$ only. Hyper-parameters $\gamma,\alpha\geq0$. }
\ENSURE{Updated model parameters $\hat{\theta}$.}
\STATE Randomly split latent domains to two sets: $\calS_1 \cup \calS_2=[k],\calS_1 \cap \calS_2=\emptyset$.
\STATE Meta-train 1: $\nabla_1=\frac{1}{|\calS_1|}\sum_{r\in\calS_1}\nabla_{\theta}\calL_s(r) $, where $\calL_s(r)$ is defined in Eq.~\eqref{eq:loss_local_each_both}.
\STATE Meta-train 2: $\nabla_2=\frac{1}{|\calS_2|}\sum_{r\in\calS_2}\nabla_{\theta}\calL_s(r)$.
\STATE Learn class-dependent and class-independent feature-extraction, and eliminate domain-differences: 
\begin{equation}\label{eq:opt_neg_post_meta}\small
\begin{split}
&\min_{\theta} \ \calL_{cd} + \calL_{ac} +\calL_{u} + \calL_{meta},\\
    \calL_{meta}=& \frac{\gamma}{2}\biggl[\frac{1}{|\calS_1|}\sum_{r\in\calS_1}\calL_{s}(r;\theta-\alpha\nabla_2) \\
    &+\frac{1}{|\calS_2|}\sum_{r\in\calS_2}\calL_{s}(r;\theta-\alpha\nabla_1) 
    \biggr].
\end{split}
\end{equation}
\end{algorithmic}}
\end{algorithm}
 
We follow Li~\etal~\cite{li2018learning} to analyze our Algorithm~\ref{alg:MLPU} via first order Taylor expansion, and obtain similar results: 
\begin{observation}\label{th:meta}
The meta learning loss $\calL_{meta}$ in Eq.~\eqref{eq:opt_neg_post_meta} of Algorithm~\ref{alg:MLPU} can be approximated as:
\begin{equation}\label{eq:opt_neg_post_meta_approx}\small
\calL_{meta} \approx \frac{\gamma}{2}\biggl[\frac{1}{|\calS_1|}\sum_{r\in\calS_1}\calL_{s}(r)+\frac{1}{|\calS_2|}\sum_{r\in\calS_2}\calL_{s}(r)
    \biggr] - \gamma\alpha\nabla_1\trans\nabla_2.
\end{equation}
\end{observation}
\begin{proof}
The proof is in Appendix~\ref{sec:proof_meta}.
\end{proof}
Observation~\ref{th:meta} suggests that our meta-learning based unification encourages the concordance of gradients between different domains.
Since the gradient of a deep neural network contains the information of all the related layers, it is possible to obtain superior transfer performance, comparing with Step~\ref{eq:post_merg} of Algorithm~\ref{alg:DA-UA} which only encourages the class prediction to be consistent across domains.

\section{Experiments}\label{sec:exp}

In this section, we evaluate our proposed framework. Both synthetic and real-world data sets are used for extensive evaluations. Our implementation uses Keras with Tensorflow~\cite{abadi2016tensorflow} backends.


We evaluate in the setting proposed by Liang~\etal~\cite{liang2019additive}:
training domains do not share classes, and testing combinations of $\langle$class, domain$\rangle$ are different from those of training data. We conduct comprehensive evaluations on three data sets used by Liang~\etal~\cite{liang2019additive}: (1) the C-MNIST data set~\cite{lu2018attribute} with $10$ classes and the background color as the domain-difference, (2) the re-organized CelebA data set~\cite{liu2015deep} with $211$ classes and whether wearing eyeglasses as the domain-difference, and (3) the authentication data set based on mobile sensors developed by Liang~\etal~\cite{liang2019additive} with $29$ classes and the OS types as the domain-difference. The detailed re-organization process for each data set has been described by Liang~\etal~\cite{liang2019additive} and will be introduced in the following sections. Domain labels are not used. 10\% data of the testing set are randomly selected for validation. For each method on each data set, we repeat 20 times and report the averaged results.

\noindent\textbf{Methods for Comparison \;} We compare our method in Algorithm~\ref{alg:DA-UA} (Ours) and our method replacing the Step~\ref{eq:post_merg} of Algorithm~\ref{alg:DA-UA} with Algorithm~\ref{alg:MLPU} (Ours-Meta) with the state-of-the-art baselines by firstly comparing the direct learning strategy (Direct) that stacks $P, G_{cd}$, and $R_{g, cd}$ only. Then we compare the methods that use extra domain labels to evaluate to what extent the provided domain labels can help to boost the transfer performances. These methods include ABS-Net~\cite{lu2018attribute},
ELEGANT~\cite{Xiao_2018_ECCV},
RevGrad~\cite{ganin2014unsupervised},   CDRD~\cite{liu2017detach},
SE-GZSL~\cite{verma2018generalized},
and AAL-UA~\cite{liang2019additive}. We present their results reported by Liang~\etal~\cite{liang2019additive}.  For the methods that do not use extra domain labels for training, we compare three methods that implicitly discover and unify latent domains: MCD~\cite{saito2018maximum}, MCD-SWD~\cite{lee2019sliced}, and ML-VAE~\cite{bouchacourt2017multi}; and two methods that explicitly perform latent-domain discovery-and-unification: mDA~\cite{mancini2018boosting} and CmDA~\cite{mancini2019discovering}. For the domain-agnostic domain-adaptation methods, we compare DADA~\cite{peng2019domain} and BTDA~\cite{chen2019blending} only, because that mDA can be regarded as an improved version of DANL~\cite{romijnders2019domain} for our problem, and that OCDA~\cite{liu2019compound} is elaborate, whose code has not been released yet; for each data set, we treat the training data as both the required source and target domains for these methods to train. For image data sets, we involve the self-supervised methods JiGen~\cite{carlucci2019domain},  Rot~\cite{xu2019self} and MAXL~\cite{liu2019self}. We build the base modules (\eg, feature extractors, and classifiers) with the same structure as ours to conduct fair experiments with the hyper-parameters optimized on the data sets and settings in this paper.

\noindent\textbf{Evaluation Metrics \;}
We follow the settings of Liang~\etal~\cite{liang2019additive} to evaluate prediction performances for both multi-label and multi-class types of recognition. For the multi-label type, we use average AUC (aAUC) which is defined as the average of the area under the ROC curve for every class, the average false acceptance rate (aFAR), and the average false rejection rate (aFRR). We report aAUC and average balanced false rate (aBFR $ =(\mbox{aFAR}+\mbox{aFRR})/2$) as balanced scores since the negative samples dominate for each class. For the multi-class type, we report top-1 accuracy (ACC@1).

\noindent\textbf{Implementation Details \;}
We constrain our model-capacity to be the same as Liang~\etal~\cite{liang2019additive} to acquire fair comparisons.
For all experiments, $G_{cd}$ and $G_{ci}$ are built by a single hidden layer with hyperbolic-tangent activation function, respectively. $R_{g,cd},R_{g,ci},\{R_{l,cd}^r\}_{r=1}^k,\{R_{l,ci}^r\}_{r=1}^k,D_{cd}$ and $D_{ci}$ are built by generalized linear layers, and using softmax as the activation function. 
We build input mapping networks $P$s with the same structures as those designed by Liang~\etal~\cite{liang2019additive} to acquire fair comparisons.
For image data sets, a simple Convolutional Neural Network (CNN) shown in Table~\ref{tab:simple_CNN} is built as the network $P$. A dropout layer with drop-rate of $0.25$ is used after the max-pooling layer. The output of $P$ is flattened for subsequent full-connected layers.
\setlength{\tabcolsep}{2pt}
\begin{table}[htpb]\small
  \caption{The CNN model used as $P$ for image data sets.}\label{tab:simple_CNN}
  \centering
  \begin{tabular}{cccccc}
    \toprule
      Layer & \# Filters &  Kernel Size & Stride & \# Padding &Activation \\
     \midrule
 Convolution & 32 &  3 & 3 & 0 & ReLu\\ 
 Convolution & 32 &  3 & 3 & 0 & ReLu\\ 
 Max-pooling & - &  2 & 2 & 0 & -\\
    \bottomrule
  \end{tabular}
\end{table}
For vector-based data sets, $P$ is built by a fully-connected neural network as shown in Table~\ref{tab:input_fc}. A dropout layer with drop-rate of $0.5$ is used after the max-pooling layer. The output of $P$ is flattened for subsequent full-connected layers.
\setlength{\tabcolsep}{2pt}
\begin{table}[htpb]\small
  \caption{The fully-connected model used as $P$ for vector-based data sets.}\label{tab:input_fc}
  \centering
  \begin{tabular}{cc}
    \toprule
      Layer & \# Filters  \\
     \midrule
 Fully-Connected & 512  \\ 
 Batch-Normalization & - \\ 
 Swish Activation~\cite{ramachandran2018searching} & -\\
    \bottomrule
  \end{tabular}
\end{table}
We set $k=2$ as discussed in the introduction. For the C-MNIST and the Mobile data set, the batch size is set to $b=512$, while for the CelebA data set, $b=128$. 
We approximate the meta learning loss $\calL_{meta}$ by Eq.~\eqref{eq:opt_neg_post_meta_approx}, according to Observation~\ref{th:meta}, and set $\gamma=0.01,\alpha=1$.


\subsection{Handwritten Digital Experiments}

The C-MNIST data set is originally built by Lu \etal~\cite{lu2018attribute}. It consists of $70$k colored RGB digital images with resolution of $28\times 28$ ($60$k for training and $10$k for testing). It is built from the original gray images of MNIST by adding $10$ background colors (b-colors) and other $10$ foreground colors (f-colors), resulting in 1k possible combinations ($10$ digits $\times$ $10$ b-colors $\times$ $10$ f-colors). Examples from C-MNIST are shown in Fig.~6 of Lu \etal~\cite{lu2018attribute}. Liang~\etal~\cite{liang2019additive} re-construct the C-MNIST data set for their GCDR problem. As shown in Fig.~\ref{fig:cmnist_dataset_bg}, training digits $0\sim 4$ have a green b-color, while $5\sim 9$ have a pink b-color. On the contrary, testing digits $0\sim 4$ have a pink b-color, while $5\sim 9$ have a green b-color. Other data are dropped. The re-construction results in $5970$ training instances and $1003$ testing instances in total.

Table~\ref{tab:result_basic_compare} summarizes the performance comparisons on C-MNIST. The results clearly show that our methods significantly outperform the direct learning method, which proves the effectiveness of our methods. Furthermore, our methods significantly outperform the baseline methods that do not use domain labels for training, which shows the superiority of our proposed methods. Moreover, our methods even significantly outperform the majority of the baseline methods that use domain labels for training, except for the AAL-UA and the SE-GZSL methods. Considering these methods use additional domain labels, which renders it much easier for cross-domain recognition, these results provide sound evidence for the effectiveness and superiority of our methods. 
Fig.~\ref{fig:cmnist_exp} shows some examples of generated hidden feature maps, in which the background difference is eliminated. Our meta-learning extension obtains comparable performances, compared with our best results, which demonstrates that our meta-learning framework is promising to handle our GCLDR problem.
On the other hand, the MCD based methods show severe negative transfer. We conjecture that the training objective of MCD is not suitable for our GCLDR problem. 
Detailed discussion will be present in Section~\ref{sec:discussion}.

\begin{table}
  \caption{Performances (\%) comparison on the C-MNIST data set. ``$^{*}$'' denotes the methods that use \emph{extra} domain labels for training.}\label{tab:result_basic_compare}
  \centering
  \begin{tabular}{lccc}
    \toprule
     Methods  & aAUC  & aBFR & ACC@1 \\
     \midrule
 Direct & 78.67     & 26.32& 20.88\\\hline
 $^{*}$ABS-Net~\cite{lu2018attribute} & 77.69    &27.41& 15.92 \\
 $^{*}$ELEGANT~\cite{Xiao_2018_ECCV}    &  79.94    & 24.61&10.68 \\
 $^{*}$RevGrad~\cite{ganin2014unsupervised}     & 80.71  &24.45 & 21.68 \\
 $^{*}$CDRD~\cite{liu2017detach} & 84.83  & 35.79  & 33.49 \\
 $^{*}$AAL-UA~\cite{liang2019additive} &  98.42  & 6.14 & 84.27 \\
 $^{*}$SE-GZSL~\cite{verma2018generalized} &   {99.79}    & {2.72}&  {94.83} \\\hline
 MCD~\cite{saito2018maximum}   &49.90    & 50.09 & 10.89\\
 MCD-SWD~\cite{lee2019sliced}   &50.12    & 49.89 & 10.69\\
 ML-VAE~\cite{bouchacourt2017multi}  &77.26    & 28.06 & 18.73\\ 
 DADA~\cite{peng2019domain} &83.90    & 22.29 & 15.83\\
 BTDA~\cite{chen2019blending}  &85.14    & 20.82 & 26.43\\
 JiGen~\cite{carlucci2019domain}   &82.44   & 24.44 & 33.33\\
 Rot~\cite{xu2019self}   &75.52   & 36.41 & 18.42\\
 MAXL~\cite{liu2019self} & 78.87   & 25.28 & 21.31\\
 mDA~\cite{mancini2018boosting} &83.34  & 21.98 & 24.26\\
 CmDA~\cite{mancini2019discovering} &86.52  & 21.00 & 43.21\\\hline
Ours  &  \textbf{93.46}  & \textbf{13.26} & \textbf{60.63} \\
Ours-Meta  &  91.06  & 16.13 & 53.16 \\
    \bottomrule
  \end{tabular}
\end{table}

\begin{table}
  \caption{Performances (\%) comparison on the CelebA data set. ``$^{*}$'' denotes the methods that use \emph{extra} domain labels for training.}\label{tab:result_celebA_compare}
  \centering
  \begin{tabular}{lccc}
    \toprule
     Methods  & aAUC  & aBFR & ACC@1 \\
     \midrule
 Direct & 78.74    & 41.58& 11.49 \\\hline
 $^{*}$ABS-Net~\cite{lu2018attribute} & 75.80    &34.90& 8.09\\
 $^{*}$ELEGANT~\cite{Xiao_2018_ECCV}    &  75.88  & 32.02&  10.05\\
 $^{*}$RevGrad~\cite{ganin2014unsupervised}     & 80.12   &  31.18& 10.96\\
$^{*}$CDRD~\cite{liu2017detach} & 80.20    &   39.90&  {16.47}\\
$^{*}$SE-GZSL~\cite{verma2018generalized}  & 84.96   &26.62 & 12.76
\\
 $^{*}$AAL-UA~\cite{liang2019additive} &   {87.07}   & {22.19}&  {14.99}\\\hline
 MCD~\cite{saito2018maximum}      & 50.16    &49.98& 0.45 \\
 MCD-SWD~\cite{lee2019sliced}      & 50.23    &50.15& 0.37 \\
 ML-VAE~\cite{bouchacourt2017multi}     & 75.29    &36.07& 7.97 \\
 DADA~\cite{peng2019domain}     & 83.37    &28.06& 11.36 \\
 BTDA~\cite{chen2019blending}  & 78.23  & 28.53    & 8.30  \\
  JiGen~\cite{carlucci2019domain}   &81.94   & 31.75 & 10.57\\
 Rot~\cite{xu2019self}   &76.02  & 34.37 & 7.55\\
 MAXL~\cite{liu2019self} & 81.02   & 28.26 & 11.05\\
 mDA~\cite{mancini2018boosting} &80.19  & 28.91 & 10.90\\
 CmDA~\cite{mancini2019discovering} &83.57  & 27.98 & 11.79\\
 \hline
 Ours &  \textbf{88.52}   & \textbf{22.74}&  \textbf{22.31}\\
  Ours-Meta &   {85.41}   &  {25.03}&  {15.61}\\
    \bottomrule
  \end{tabular}
\end{table}

\subsection{Face Recognition}

We use the aligned, cropped and scaled version of the CelebA data set~\cite{liu2015deep} with image-size of $64\times 64$. Liang~\etal~\cite{liang2019additive}  chose the \textit{Eyeglasses} attribute as the domain-difference, selected individuals with at least $20$ images, and balanced the data set such that $\#(Eyeglasses=0)/\#(Eyeglasses=1)\in[3/7,7/3]$, resulting in $211$ individuals.
Half of the individuals wear glasses only during training, while the other half wear glasses only during testing. 
Table~\ref{tab:result_celebA_compare} shows the comparisons conducted on CelebA. We achieve consistent results with those in Table~\ref{tab:result_basic_compare}. Our methods significantly outperform the baseline methods without domain-label supervision and most baseline methods supervised by extra domain labels, which demonstrates the effectiveness and superiority of our methods. {Note that our method even outperforms the best (AAL-UA) of the methods with domain-label supervision.}

\subsection{Authentication on Mobile Devices}\label{subsubsec:PIE}

We use the mobile data set built by Liang~\etal~\cite{liang2019additive} who collect smart-phone sensor information from 29 subjects, which records two-second time-series data from multiple sensors, such as accelerometer, gyroscope, gravimeter, \etc. They extracted statistical features from both time and spectrum domains, resulting in 5144 data samples with the feature dimension of 191. They treated the OS types (IOS/Android) as the domain-difference and constructed a biased learning task, as shown in Table~\ref{tab:mobile_problem}. The results are reported in Table~\ref{tab:result_mobile_compare}, in which our methods still achieve consistent results. We can see that our methods significantly outperform the baseline methods without domain-label supervision and most baseline methods supervised by domain labels.
On this data set, our meta-learning method outperforms our baseline method, which demonstrates the effectiveness of our extended framework.
Note that the results of MCD based methods do not show significant negative transfer here, we conjecture that it is because over-fitting the domain-difference in this data set (OS type) is not very straightforward, compared with learning to recognize individuals. We will also discuss this phenomenon in Section~\ref{sec:discussion}.
 
\begin{table}

  \caption{The authentication problem on mobile devices. The numbers in the first row indicate groups of subjects. ``$ \times$'' means there are no data for this condition.}\label{tab:mobile_problem}
  \centering
  \begin{tabular}{ccccc}
    \toprule
      & No. 1-6 & No. 7-12 & No. 13-15&No. 16-29\\
     \midrule
 IOS & Train   & Test & $ \times$ & Train\\ 
Android & Test  & Train  & Train&  $ \times$\\
    \bottomrule
  \end{tabular}
  \end{table}

 \begin{table} 
  \caption{Performances (\%) comparison on the Mobile data set. ``$^{*}$'' denotes the methods that use \emph{extra} domain labels for training.}\label{tab:result_mobile_compare}
  \centering
  \begin{tabular}{lccc}
    \toprule
     Methods  & aAUC  & aBFR & ACC@1 \\
     \midrule
 Direct & 76.53    & 28.64& 3.79  \\\hline
 $^{*}$RevGrad~\cite{ganin2014unsupervised}     & 75.88   &32.38& 0.38 \\
 $^{*}$ABS-Net~\cite{lu2018attribute} &76.58     & 28.09& 5.13\\
 $^{*}$SE-GZSL~\cite{verma2018generalized}  & 78.83   &26.12& 20.54  \\
 $^{*}$CDRD~\cite{liu2017detach} & 89.17    &20.26& 46.05 \\
 $^{*}$AAL-UA~\cite{liang2019additive} &   {93.40}    &  {13.59}&   {46.37}\\
 \hline
 MCD~\cite{saito2018maximum}     &83.12    & 21.37& 24.35  \\
 MCD-SWD~\cite{lee2019sliced}     &84.49    & 19.89& 25.35  \\
 ML-VAE~\cite{bouchacourt2017multi}     &77.16    & 27.18& 4.68  \\
 DADA~\cite{peng2019domain}     &77.77   & 27.87& 22.40  \\
 BTDA~\cite{chen2019blending}     &86.96    & 17.85& 30.21  \\
 MAXL~\cite{liu2019self} & 77.02   & 27.26 & 5.05\\
 mDA~\cite{mancini2018boosting} &81.43  & 26.38 & 18.80\\
 CmDA~\cite{mancini2019discovering} &82.22  & 21.84 & 20.42\\\hline
 Ours &   {90.72}    &  {16.04}&   {35.49}\\
  Ours-Meta &  \textbf{91.55}    & \textbf{14.87}&  \textbf{35.84}\\
    \bottomrule
  \end{tabular}
 \end{table}

 
  
\subsection{Ablative Study}\label{subsec:ablation}
 
We conduct a series of ablation experiments on the three data sets mentioned above to demonstrate how the heterogeneous predictive relationship discovery and the double-space domain-difference elimination mechanisms contribute to the performance. 
Specifically, we compare the following four model variants of our method.

\textit{Single-Space}. We learn and unify latent domains in the class-dependent space only, \ie, the branches of networks for the class-independent space are removed, \ie, $G_{ci},R_{g,ci},\{R_{l,ci}^r\}_{r=1}^k$ and $D_{ci}$.


\textit{Feature-Based}. We learn and unify latent domains based on features only, \ie, local recognition networks $\{R_{l,cd}^r\}_{r=1}^k$ and $\{R_{l,ci}^r\}_{r=1}^k$ are removed.

 
\textit{Class-Confuse}. No latent-domain discovery and unification. But we still learn the class-independent space. Specifically, $\{R_{l,ci}^r\}_{r=1}^k,D_{cd},\{R_{l,ci}^r\}_{r=1}^k$ and $D_{ci}$ are removed.


\textit{No-Unification}. 
We discover latent domains but do not unify domains. Instead, for a testing sample, we select the recognition model from the most relevant domain to make recognition. 

Specifically, $R_{g,cd},G_{ci},R_{g,ci},\{R_{l,ci}^r\}_{r=1}^k$ and $D_{ci}$ are removed. We only use $\{R_{l,cd}^r\}_{r=1}^k$ and $D_{cd}$ to make recognition for the $i$th sample: 
\begin{equation}\label{eq:opt_pos}\small
    \hat{j} = \argmax_j \sum_{r=1}^kp(z^i=r \mid \f_{cd}^i,D_{cd})p(y^i=j \mid \f_{cd}^i,R_{l,cd}^{r}).
\end{equation}

  
\setlength{\tabcolsep}{2pt}
\begin{table}[!t]\small
  \caption{Performances (\%) comparison on the C-MNIST data set for different variants of our method.}\label{tab:result_ab_basic_compare}
  \centering
  \begin{tabular}{lccc}
    \toprule
      Methods  & aAUC  & aBFR & ACC@1 \\
     \midrule
 Ours  &  \textbf{93.46}  & \textbf{13.26} & \textbf{60.63} \\\hline
 Single-Space & 86.51   &21.81& 39.76 \\
 Feature-Based   &  83.80    & 23.33&35.95 \\
 Class-Confuse     & 77.48  &28.54 & 24.22 \\
 No-Unification & 50.64  & 57.93  & 0.02 \\\hline
 Direct & 78.67     & 26.32& 20.88\\
    \bottomrule
  \end{tabular}
\end{table}
 
\setlength{\tabcolsep}{2pt}
\begin{table}[!t]\small
  \caption{Performances (\%) comparison on the CelebA data set for different variants of our method.}\label{tab:result_ab_celeba_compare}
  \centering
  \begin{tabular}{lccc}
    \toprule
      Methods  & aAUC  & aBFR & ACC@1 \\
     \midrule
  Ours &  \textbf{88.52}   & \textbf{22.74}&  \textbf{22.31} \\\hline
 Single-Space & 85.67   &24.05& 16.29 \\
 Feature-Based   &  86.32    & 23.88&16.35 \\
 Class-Confuse     & 75.63  &35.61 & 8.19 \\
 No-Unification & 74.49  & 33.90  & 7.86 \\\hline
 Direct & 78.74    & 41.58& 11.49\\
    \bottomrule
  \end{tabular}
\end{table}
 
\setlength{\tabcolsep}{2pt}
\begin{table}[!t]\small
  \caption{Performances (\%) comparison on the Mobile data set for different variants of our method.}\label{tab:result_ab_moblie_compare}
  \centering
  \begin{tabular}{lccc}
    \toprule
      Methods  & aAUC  & aBFR & ACC@1 \\
     \midrule
  Ours &   \textbf{90.72}    &  \textbf{16.04}&   \textbf{35.49} \\\hline
 Single-Space & 88.42   &17.70& 33.07 \\
 Feature-Based   &  86.33    & 18.57&25.09 \\
 Class-Confuse     & 75.13   &30.19 & 5.55 \\
 No-Unification & 72.36  & 34.71  & 4.06 \\\hline
 Direct & 76.53    & 28.64& 3.79\\
    \bottomrule
  \end{tabular}
\end{table}

The results are presented in Table~\ref{tab:result_ab_basic_compare}$\sim$\ref{tab:result_ab_moblie_compare}. 
It is notable that Feature-Based's performances drastically decrease comparing with our best scores. The results demonstrate that using the class label to model the predictive relationship $p(y\mid \x)$ can harness more information, make a more accurate latent-domain discovery, and result in better class-alignment across latent domains. Besides, Our method outperforms Single-Space significantly, which shows that it is more robust to eliminate domain-difference both in the class-dependence and class-independence spaces. Liang~\etal~\cite{liang2019additive} also finds the effectiveness of such a multi-branch structure.  Moreover, the results of Class-Confuse show that learning features in the class-independence space itself cannot contribute to cross-domain recognition, but rather slightly compromises the recognition performances comparing with the Direct method, because after all, its objective function is contrary to that of recognition. Lastly, the results of No-Unification are significantly worse than the Direct method, especially for the C-MNIST data set. It demonstrates that for our GCLDR problem when we only discover latent domains but do not unify them, a testing sample can only find a poorly trained recognition model from its domain, especially when latent domains are easy to discover. 


\section{Discussion}\label{sec:discussion}  
 
In Table~\ref{tab:result_basic_compare} and~\ref{tab:result_celebA_compare},  the results of MCD~\cite{saito2018maximum} and its updated version MCD-SWD~\cite{lee2019sliced} show severe negative transfer. We conjecture that it is because MCD admits a ``global'' optimization objective: its two classifiers are required to classify all the samples and classes in training data. MCD first (Step 1) trains the two classifiers such that they are different but still required to correctly recognize all the samples and classes. Then (Step 2) MCD trains the feature-extractors to minimize the output-distribution discrepancy of the two classifiers. In domain-adaptation problems, Step 1 and 2 perform on the source and target domains, respectively, thus are not necessarily contradictory. However, in our GCLDR problem, the two steps can only perform on the same training data. Therefore, it is possible that when the two classifiers are significantly different, the feature-extractors will achieve a trivial solution: the outputs (class prediction probabilities) of both classifiers are equal to be $1/c$ for each class and every sample, where $c$ is the number of class. We checked the classifier outputs of MCD and MCD-SWD for the C-MNIST and CelebA data sets (where negative transfer occurs), and truly found that they achieved such a trivial solution. Moreover, we conjecture that when domain-difference is easier to learn than class-difference, since domain-difference enlarges class-difference in our GCLDR problem, the two classifiers are easier to be significantly different, thus the negative transfer will be more likely to happen. We think it is the reason why the negative transfer happened only on the C-MNIST and CelebA data sets, because color and eyeglasses are easier to discriminate than digits and individuals, respectively. By contrast, our methods learn $k$ classifiers ``locally'' to recognize the samples/classes in each domain only, thus will mainly concentrate on class-difference instead of domain-difference, which may help to avoid the objective-contradiction and the negative transfer.

\section{Conclusion}\label{sec:conclusion}
In this paper, we investigate a generalized cross-latent-domain recognition problem in the field of authentication where domain labels are absent, and domains do not share classes. We recognize the class for unseen $\langle$class, domain$\rangle$ combinations of data. We propose an end-to-end domain agnostic method to tackle the problem. We build a heterogeneous predictive-relationship discovery and unification mechanism to discover and unify latent domains successively. Besides, we build a double-space domain-difference elimination mechanism to eliminate domain-difference in both class-dependent and class-independent spaces to improve robustness of elimination. We also extend our method into a meta-learning framework as an alternative elimination approach. The experiments demonstrate that our method significantly outperforms existing state-of-the-art methods. We also conduct an ablation study to demonstrate the effectiveness of the critical components of our method. 
Some interesting future directions of research include developing transfer learning algorithms flexible to emerging types of domain-difference.

{



\renewcommand{\theequation}{\thesection.\arabic{equation}}
\renewcommand{\thefigure}{\thesection.\arabic{figure}}
\renewcommand{\thepage}{\thesection.\arabic{page}}
\renewcommand{\thesection}{\thesection.\arabic{section}}
\renewcommand{\thetable}{\thesection.\arabic{table}}

\appendices

\section{Derivation Details of Posterior Probabilities and EM Procedures}\label{sec:post_and_em}

Recall that by Eq.~(1), $k$ local recognition networks $R_{l}^{1},\ldots,R_{l}^{k}$ aim to learn
\begin{equation}\label{eq:cond_supp}\small
p(y^i \mid \f^i,R_{l}^{r}) = \prod_{j=1}^cp(y^i=j \mid \f^i,R_{l}^{r})^{I(y^i=j)},r\in[k],
\end{equation}
and the domain-discrimination network $D$ aims to learn $p(z^i=r\mid \f,D)$.
Denote by $\theta = \{ \{R_l^{r}\}_{r=1}^k,D \}$.
Then by the Bayes' rule, we have 
\begin{equation}\label{eq:post_supp}\small
\begin{split}
p(&   z^i=r \mid y^i,\f^i,\{R_l^{r}\}_{r=1}^k,D)
=p(   z^i=r \mid y^i,\f^i,\theta) \\
&= \frac{p(   z^i=r, y^i \mid \f^i,\theta)}{
p(    y^i \mid \f^i,\theta)
} \\
&= \frac{p(   z^i=r, y^i \mid \f^i,\theta)}{
\sum_{r'=1}^kp(  z^i=r',  y^i \mid \f^i,\theta)
} \\
&= \frac{p(   z^i=r \mid \f^i,\theta)p(   y^i \mid z^i=r,  \f^i,\theta)}{
\sum_{r'=1}^kp(  z^i=r'  \mid \f^i,\theta)
p(   y^i \mid  z^i=r', \f^i,\theta)
} \\
&= \frac{p(   z^i=r \mid \f^i,D)p(   y^i \mid z^i=r,  \f^i,\{R_l^{r}\}_{r=1}^k)}{
\sum_{r'=1}^kp(  z^i=r'  \mid \f^i,D)
p(   y^i \mid  z^i=r', \f^i,\{R_l^{r}\}_{r=1}^k)
} \\
&= \frac{p(   z^i=r \mid \f^i,D)p(   y^i \mid   \f^i,R_l^{r})}{
\sum_{r'=1}^kp(  z^i=r'  \mid \f^i,D)
p(   y^i \mid    \f^i,R_l^{r'})
} \\
 &= \frac{p( z^i=r \mid \f^i,D)\prod_{j=1}^cp(y^i=j \mid \f^i,R_l^{r})^{I(y^i=j)}}{\sum_{r'=1}^kp(z^i=r' \mid \f^i,D)\prod_{j=1}^cp(y^i=j \mid \f^i,R_l^{r'})^{I(y^i=j)}}.
\end{split}
\end{equation}
The third equation is due to the law of total probability.
The fifth equation is because that given $\f^i$,  the prediction for $z^i=r$ is related to $D$ only, and given $\f^i$ and $z^i=r$, the prediction for $y^i$ is related to $\{R_l^{r}\}_{r=1}^k$ only. The sixth equation is because that given $z^i=r$, the prediction for $y^i$ is related to $R_l^{r}$ only. The last equation is the result of Eq.~\eqref{eq:cond_supp}.

We follow the Expectation-Maximization (EM)~\cite{dempster1977maximum} scheme to solve the problem. For each $i=[b]$, define $(\delta^{i,1},\ldots,\delta^{i,k})$ be a set of latent indicator variables, where $\delta^{i,r} = 1$ if the $i$th sample $(y^i,\f^i)$ belongs to the $r$th latent domain and $\delta^{i,r} = 0$ otherwise. So $\sum_{r=1}^k\delta^{i,r}=1$, $\forall i$. These indicators are not observed since the domain labels of the samples are unknown. Let $\delta$ denote the collection of all the indicator variables. By treating $\delta$ as missing, the EM algorithm proceeds by iteratively optimizing the conditional expectation of the complete log-likelihood criterion.

The complete likelihood is given by
\begin{equation}\label{eq:complete_ll_supp}\small
\begin{split}
 \prod_{i=1}^b\prod_{r=1}^k\prod_{j=1}^c[
p( z^i=r \mid \f^i,D)p(y^i=j \mid \f^i,R_l^{r})
]^{\delta^{i,r}I(y^i=j)}.
\end{split}
\end{equation}
Then the complete log-likelihood is given by
\begin{equation}\label{eq:complete_logll_supp}\small
\begin{split}
&\ell_c(\theta\mid \{(y^i,\f^i)\}_{i=1}^b,\delta) \\
&=  \sum_{i=1}^b\sum_{r=1}^k\sum_{j=1}^c
\delta^{i,r}I(y^i=j)\log
p( z^i=r \mid \f^i,D)\\
&+\sum_{i=1}^b\sum_{r=1}^k\sum_{j=1}^c
\delta^{i,r}I(y^i=j)\log
p(y^i=j \mid \f^i,R_l^{r}),
\end{split}
\end{equation}
where $\theta = \{ \{R_l^{r}\}_{r=1}^k,D \}$.
The conditional expectation of the complete negative log-likelihood is then given by
\begin{equation}\label{eq:complete_Q_supp}\small
\begin{split}
Q(\theta\mid \theta') =  - \mathbb{E}[\ell_c(\theta\mid \{(y^i,\f^i)\}_{i=1}^b,\delta)\mid\{(y^i,\f^i)\}_{i=1}^b,\theta']/b,
\end{split}
\end{equation}
It is easy to show that deriving $Q(\theta\mid \theta')$ boils down to the computation of $\mathbb{E}[\delta_{i,r}\mid\{(y^i,\f^i)\}_{i=1}^b,\theta']$, which admits an explicit form.

The EM algorithm proceeds as follows. 

\noindent\textbf{E-Step:} Given $\theta' = \{ \{R_l^{r}\}_{r=1}^k,D \}$ computed by the last step of optimization, compute
\begin{equation}\label{eq:complete_rho_supp}\small
\begin{split}
\rho^{i,r}&=\mathbb{E}[\delta_{i,r}\mid\{(y^i,\f^i)\}_{i=1}^b,\theta']\\
&=p(   z^i=r \mid y^i,\f^i,\theta')\\
&=p(   z^i=r \mid y^i,\f^i,\{R_l^{r}\}_{r=1}^k,D)=\\
&\frac{p( z^i=r \mid \f^i,D)\prod_{j=1}^cp(y^i=j \mid \f^i,R_l^{r})^{I(y^i=j)}}{\sum_{r'=1}^kp(z^i=r' \mid \f^i,D)\prod_{j=1}^cp(y^i=j \mid \f^i,R_l^{r'})^{I(y^i=j)}},
\end{split}
\end{equation}
where the last equation is the result of Eq.~\eqref{eq:post_supp}.

\noindent\textbf{M-Step:} Minimize $Q(\theta\mid \theta')$
\begin{equation}\label{eq:complete_Qwithrho_supp}\small
\begin{split}
Q(\theta\mid \theta')
&=  -\frac{1}{b}\sum_{i=1}^b\sum_{r=1}^k\sum_{j=1}^c
\rho^{i,r}I(y^i=j)\log
p( z^i=r \mid \f^i,D)\\
&-\frac{1}{b}\sum_{i=1}^b\sum_{r=1}^k\sum_{j=1}^c
\rho^{i,r}I(y^i=j)\log
p(y^i=j \mid \f^i,R_l^{r}),
\end{split}
\end{equation}

a) optimize the recognition model for each domain by:
\begin{equation}\label{eq:opt_R_supp}\small
\begin{split}
&-\frac{1}{b}\sum_{i=1}^b\sum_{r=1}^k\sum_{j=1}^c
\rho^{i,r}I(y^i=j)\log
p(y^i=j \mid \f^i,R_l^{r})\\
&=-\frac{1}{b}\sum_{i=1}^b\sum_{r=1}^k\rho^{i,r}\sum_{j=1}^c
I(y^i=j)\log
p(y^i=j \mid \f^i,R_l^{r}).
\end{split}
\end{equation}
Eq.~\eqref{eq:opt_R_supp} corresponds to Eq.~(3).

b) optimize the concordance between the feature-based discriminative probability $p( z^i=r \mid \f^i,D)$ and the posterior probability $\rho^{i,r}$ by:
\begin{equation}\label{eq:opt_D_supp}\small
\begin{split}
&-\frac{1}{b}\sum_{i=1}^b\sum_{r=1}^k\sum_{j=1}^c
\rho^{i,r}I(y^i=j)\log
p( z^i=r \mid \f^i,D)\\
&=-\frac{1}{b}\sum_{i=1}^b\sum_{r=1}^k
\rho^{i,r}\log
p( z^i=r \mid \f^i,D)\sum_{j=1}^cI(y^i=j)\\
&=-\frac{1}{b}\sum_{i=1}^b\sum_{r=1}^k
\rho^{i,r}\log
p( z^i=r \mid \f^i,D).
\end{split}
\end{equation}
Eq.~\eqref{eq:opt_D_supp} corresponds to Eq.~(4).

\section{Proof of Observation~\ref{th:meta}}\label{sec:proof_meta}
\begin{proof}
By the first order Tyler expansion, we have
\begin{equation}
\begin{split}
    \calL_{meta}=& \frac{\gamma}{2}\biggl[\frac{1}{|\calS_1|}\sum_{r\in\calS_1}\calL_{s}(r;\theta-\alpha\nabla_2) \\
    &+\frac{1}{|\calS_2|}\sum_{r\in\calS_2}\calL_{s}(r;\theta-\alpha\nabla_1) 
    \biggr]\\
    \approx& 
    \frac{\gamma}{2}\biggl[\frac{1}{|\calS_1|}\sum_{r\in\calS_1}\calL_{s}(r;\theta)-\alpha(\nabla_{\theta}\calL_{s}(r;\theta))\trans\nabla_2 \\
    &+\frac{1}{|\calS_2|}\sum_{r\in\calS_2}\calL_{s}(r;\theta)-\alpha(\nabla_{\theta}\calL_{s}(r;\theta))\trans\nabla_1
    \biggr]\\
    =&\frac{\gamma}{2}\biggl[\frac{1}{|\calS_1|}\sum_{r\in\calS_1}\calL_{s}(r;\theta)-\alpha\nabla_1\trans\nabla_2 \\
    &+\frac{1}{|\calS_2|}\sum_{r\in\calS_2}\calL_{s}(r;\theta)-\alpha\nabla_2\trans\nabla_1
    \biggr]\\
    =&\frac{\gamma}{2}\biggl[\frac{1}{|\calS_1|}\sum_{r\in\calS_1}\calL_{s}(r)+\frac{1}{|\calS_2|}\sum_{r\in\calS_2}\calL_{s}(r)
    \biggr] - \gamma\alpha\nabla_1\trans\nabla_2.
\end{split}
\end{equation}
 
\end{proof}

}

%


%
%
\ifCLASSOPTIONcaptionsoff
\newpage
\fi

{\small
\bibliographystyle{abbrvnat}
\bibliography{sigproc} %
}
 
\begin{IEEEbiography}[{\includegraphics[width=1in,height=1.25in,clip,keepaspectratio]{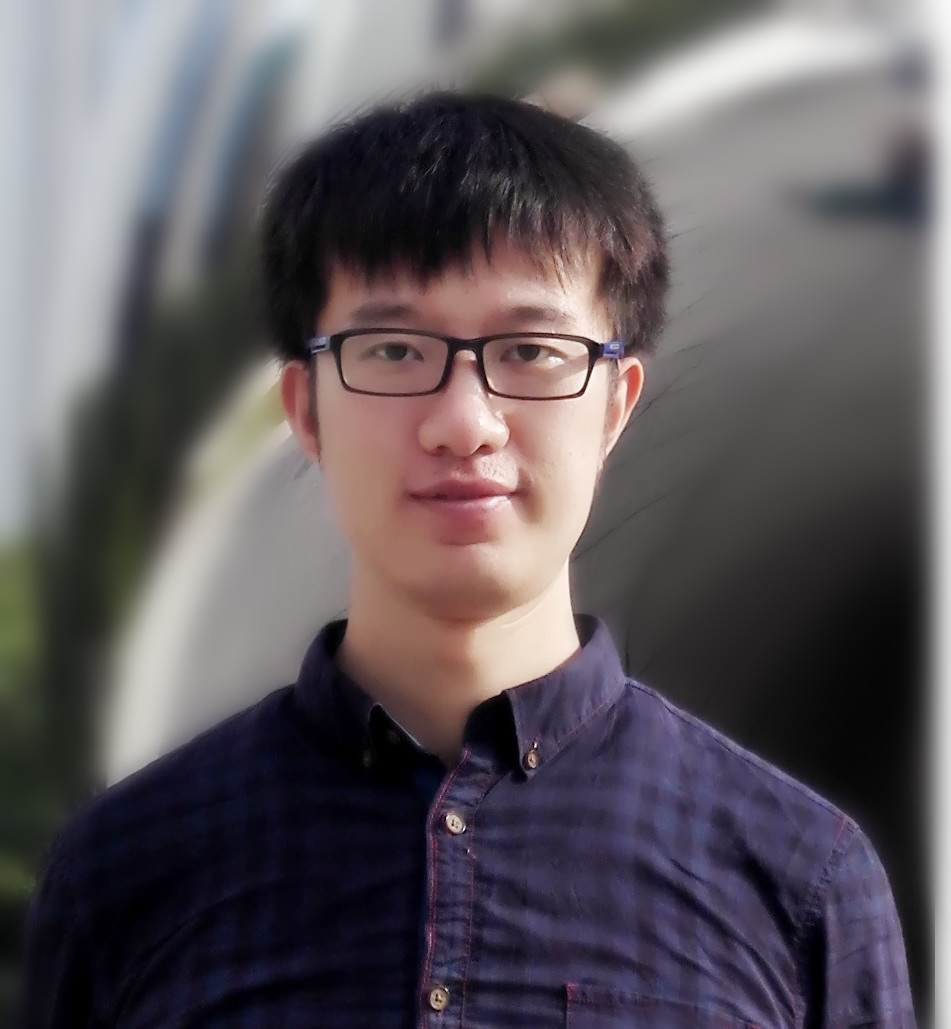}}]{Jian Liang}
Jian Liang received his Ph.D. degree from Tsinghua University, Beijing, China, in 2018. During 2018 and 2020 he was a senior researcher in the Wireless Security Products Department of the Cloud and Smart Industries Group at Tencent, Beijing. In 2020 he joined the AI for international Department, New Retail Intelligence Engine, Alibaba Group as a senior algorithm engineer. His paper received the Best Short Paper Award in 2016 IEEE International Conference on Healthcare Informatics (ICHI).
\end{IEEEbiography}

\vspace{-2em}

\begin{IEEEbiography}[{\includegraphics[width=1in,height=1.25in,clip,keepaspectratio]{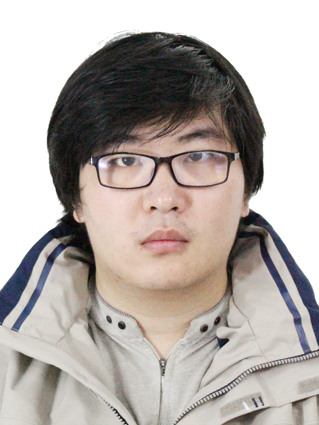}}]{Yuren Cao}
Yuren Cao received his master degree from University of Electronic Science and Technology of China, Chengdu, China,  in 2019.
He is currently a researcher in the Wireless Security Products Department of the Cloud and Smart Industries at Tencent. 
His research interests include machine learning, deep learning and data mining.
\end{IEEEbiography}

\vspace{-2em}

\begin{IEEEbiography}[{\includegraphics[width=1in,height=1.25in,clip,keepaspectratio]{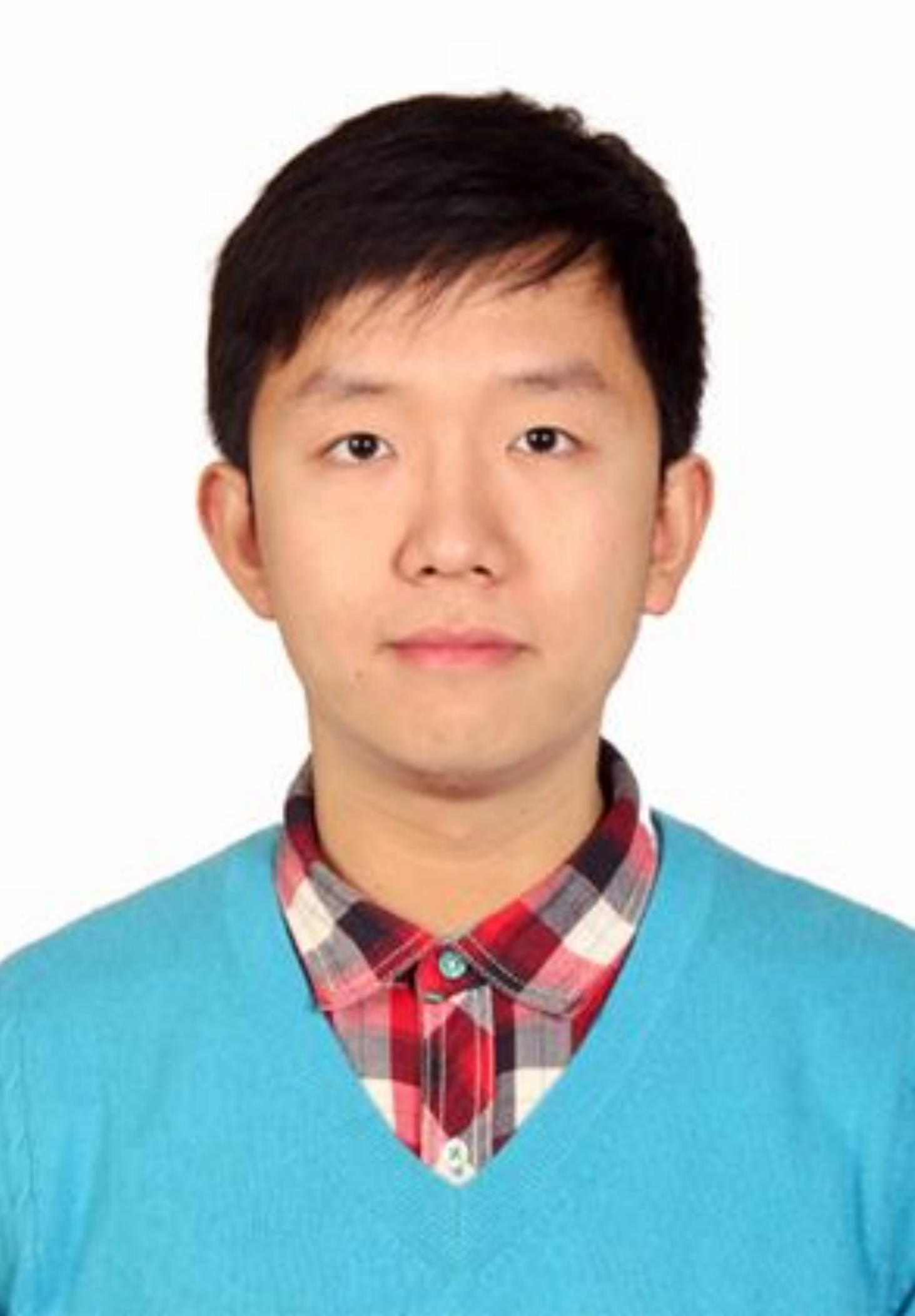}}]{Shuang Li}Shuang Li received the Ph.D. degree in control science and engineering from the Department of Automation, Tsinghua University, Beijing, China, in 2018. He was a Visiting Research Scholar with the Department of Computer Science, Cornell University, Ithaca, NY, USA, from November 2015 to June 2016. He is currently an Assistant Professor with the school of Computer Science and Technology, Beijing Institute of Technology, Beijing. His main research interests include machine learning and deep learning, especially in transfer learning and domain adaptation.
\end{IEEEbiography}

\vspace{-2em}

\begin{IEEEbiography}[{\includegraphics[width=1in,height=1.25in,clip,keepaspectratio]{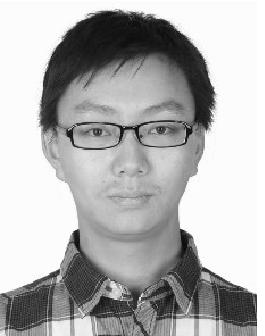}}]{Bing Bai}
Bing Bai received his B.S. and Ph.D. degrees in control theory and application from Tsinghua University, China, in 2013 and 2018 respectively, and he is currently a senior researcher with the Cloud and Smart Industries Group, Tencent, Beijing, China. His research interests include natural language processing and recommender systems.
\end{IEEEbiography}

\vspace{-2em}

\begin{IEEEbiography}[{\includegraphics[width=1in,height=1.25in,clip,keepaspectratio]{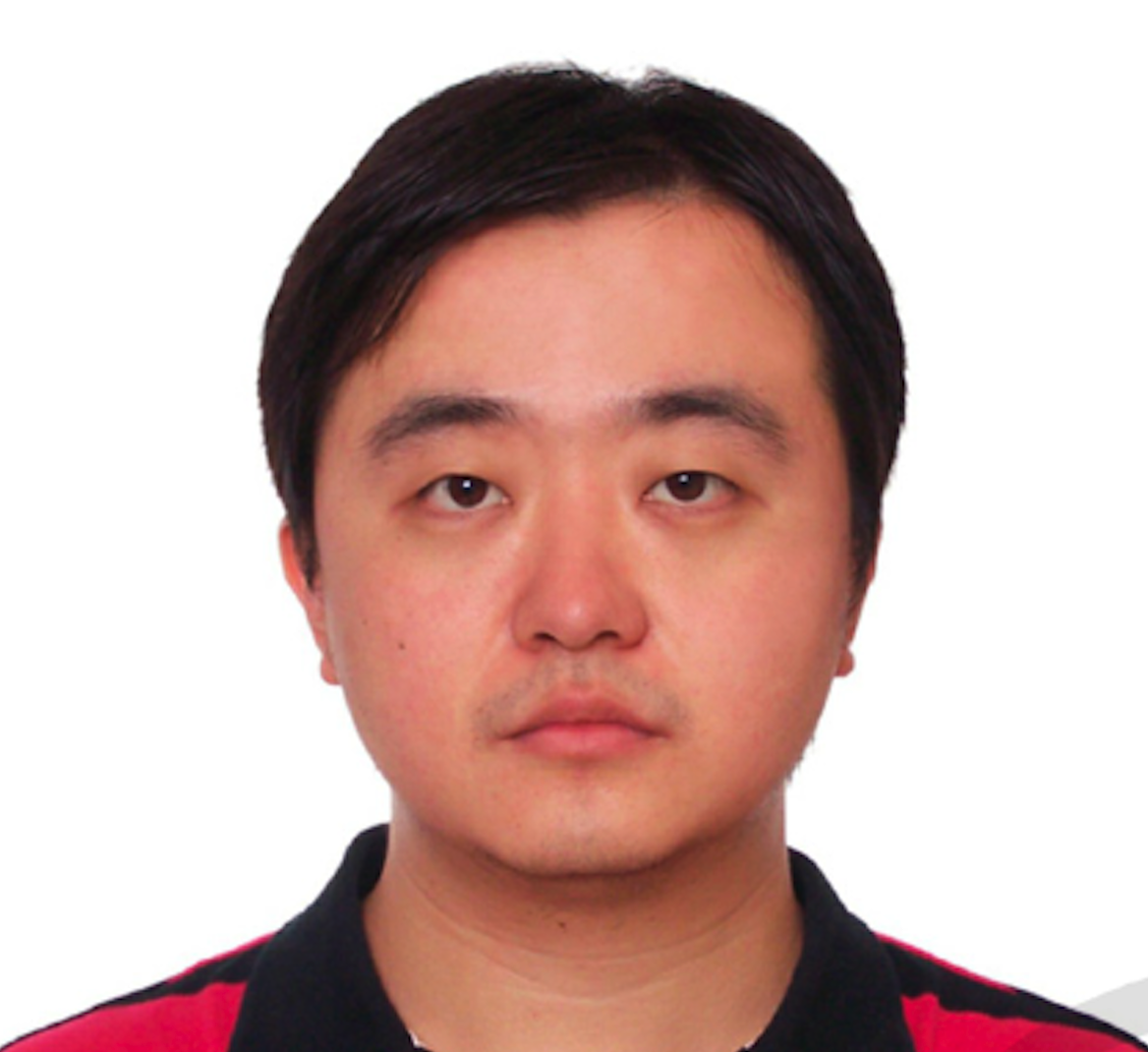}}]{Hao Li}
Hao Li is a principle researcher and engineer at Tencent. Hao’s research centers on making systems and data reliable, secure, and efficient. Recently he’s been focusing on data security and privacy by distributed and decentralized cross-silo/cross-device federated learning, and privacy-preserving machine learning by program analysis and secure multi-party computation. 
\end{IEEEbiography}
 
\vspace{-2em}

\begin{IEEEbiography}[{\includegraphics[width=1in,height=1.25in,clip,keepaspectratio]{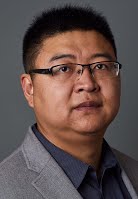}}]{Fei Wang}
Fei Wang is an Associate Professor in Division of Health Informatics, Department of Population Health Sciences, Weill Cornell Medicine, Cornell University. His major research interest is data mining, machine learning and their applications in health data science. He has published more than 200 papers on the top venues of related areas such as ICML, KDD, NeurIPS, AAAI, JAMA Internal Medicine, Annals of Internal Medicine, etc. His papers have received over 12,500 citations so far with an H-index 55. His (or his students’) papers have won 7 best paper (or nomination) awards at international academic conferences. His team won the championship of the NIPS/Kaggle Challenge on Classification of Clinically Actionable Genetic Mutations in 2017 and Parkinson's Progression Markers' Initiative data challenge organized by Michael J. Fox Foundation in 2016. Dr. Wang is the recipient of the NSF CAREER Award in 2018, as well as the inaugural research leadership award in IEEE International Conference on Health Informatics (ICHI) 2019. Dr. Wang is the chair of the Knowledge Discovery and Data Mining working group in American Medical Informatics Association (AMIA). Dr. Wang frequently serves as the program committee chair, general chair and area chair at international conferences on data mining and medical informatics. Dr. Wang is on the editorial board of several prestigious academic journals including Scientific Reports, IEEE Transactions on Neural Networks and Learning Systems, Data Mining and Knowledge Discovery, etc.
\end{IEEEbiography}

\vspace{-2em}

\begin{IEEEbiography}[{\includegraphics[width=1in,height=1.25in,clip,keepaspectratio]{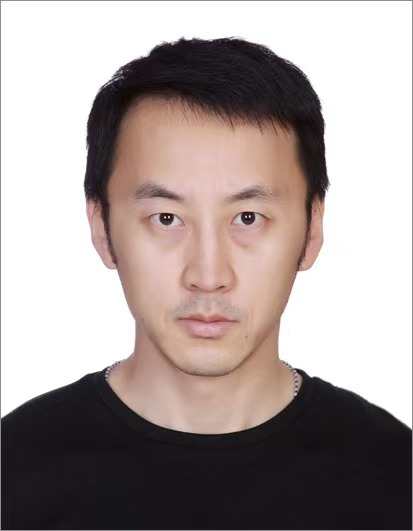}}]{Kun Bai}
Kun Bai is the Director of Cloud \& Smart Industries Group in Tencent. 
Before joining Tencent, he was Research Staff Member and Manager in IBM TJ Watson Research and was responsible for developing and leading advanced research for IBM Watson Health and IBM Watson Cloud. He earned a Ph.D. in Information Science \& Technologies from Pennsylvania State University. He is a senior member of IEEE.
\end{IEEEbiography}
%

\end{document}